\documentclass[acmsmall]{acmart}

\usepackage[all]{xy}

\usepackage{tikz}
\usepackage{hyperref}

\newtheorem{remark}{Remark}

\newcommand{\val}[1]{[\![{#1}]\!]}
\newcommand{\descr}[1]{(\![{#1}]\!)}
\newcommand{\bel}{\mathsf{bel}}
\newcommand{\pl}{\mathsf{pl}}

\newcommand{\mass}{\mathsf{m}}
\newcommand{\DS}{DS}

\newcommand{\commentSabine}[1]{}
\newcommand{\redbf}[1]{\textcolor{red}{\textbf{#1}}}

\newcommand{\redfootnote}[1]{\footnote{\textcolor{red}{\textbf{#1}}}}

\def\BibTeX{{\rm B\kern-.05em{\sc i\kern-.025em b}\kern-.08emT\kern-.1667em\lower.7ex\hbox{E}\kern-.125emX}}
    
\begin{document}

%
\title{Toward a Dempster-Shafer theory of concepts}

%
\author{Sabine Frittella}\thanks{The research of Sabine Frittella was partially funded by the grant PHC VAN GOGH 2019, project n$^\circ$42555PE and by the grant ANR JCJC 2019, project PRELAP (ANR-19-CE48-0006).}
\affiliation{%
  \institution{INSA Centre Val de Loire,
  Univ. Orl\'eans, LIFO EA 4022}
  \city{Bourges}
  \country{France}
}
\email{sabine.frittella@insa-cvl.fr}
\orcid{0000-0003-4736-8614}

\author{Krishna Manoorkar}
\affiliation{%
  \institution{Technion, Haifa, Israel}
}

\author{Alessandra Palmigiano}
\affiliation{%
  \institution{VU School of Business and Economics, Amsterdam, the Netherlands, and Department of Mathematics and Applied Mathematics, University of Johannesburg, South Africa}
}

\author{Apostolos Tzimoulis}
\affiliation{%
  \institution{VU School of Business and Economics}
  \city{Amsterdam}
  \country{The Netherlands}
}

\author{Nachoem Wijnberg}
\affiliation{%
 \institution{Amsterdam Business School, the Netherlands and College of Business and Economics, University of Johannesburg, South Africa}
 }

%
\renewcommand{\shortauthors}{Frittella Manoorkar Palmigiano Tzimoulis Wijnberg}

%
\begin{abstract}
In this paper, we  generalize the basic notions and results of  Dempster-Shafer theory from predicates to formal concepts.  Results include the representation of  conceptual belief functions as inner measures of suitable probability functions, and a Dempster-Shafer rule of combination on belief functions on formal concepts.
\smallskip

\end{abstract}

\maketitle

\section{Introduction}
\label{sec:intro}

\paragraph{Dempster-Shafer theory.}
Dempster-Shafer theory \cite{dempster1968generalization, shafer1976mathematical} is a formal framework for decision-making under uncertainty in situations in which some predicates cannot be assigned subjective probabilities. 
The core proposal of Dempster-Shafer theory is that, in such cases, the missing value can be replaced by a range of values, the lower and upper bounds of which are  assigned by {\em belief} and {\em plausibility} functions (cf.\ Definition \ref{def:bel-func}). 
This situation can  be captured with probabilistic tools via {\em probability spaces}. The later are structures $\mathbb{X} = (S, \mathbb{A}, \mu)$ where $S$ is a nonempty (finite) set, $\mathbb{A}$ 
is a $\sigma$-algebra  of subsets of $S$, and $\mu: \mathbb{A}\to [0, 1]$ is a probability measure. In probability spaces,  belief and plausibility functions  naturally arise as the  {\em inner} and {\em outer measures} induced by $\mu$ (cf.~\cite[Proposition 3.1]{fagin1991uncertainty}
). 
Moreover,  every belief function can be {\em represented} as the inner measure of some probability space (cf.~\cite[Theorem 3.3]{fagin1991uncertainty}). 
In \cite{Grabisch2008,Zhou13}, generalisations of Dempster-Shafer theory from Boolean algebras to so-called De Morgan type lattices and distributive lattices are presented.

\paragraph{Combination of beliefs.}
Central to Dempster-Shafer theory is the rule of combination of beliefs (representing e.g.~evidence, hints, or preferences) (see \cite{shafer1976mathematical}). It allows to combine possibly conflicting beliefs proceeding from multiple independent sources. Intuitively, the combined belief  according to the Dempster-Shafer  rule  highlights the converging portions of evidence,  and downplays the conflicting ones.

\paragraph{Formal concept analysis.}  Formal concept analysis (FCA) \cite{ganter2012formal} is a successful and very elegant theory in data analysis based on algebraic and lattice-theoretic facts.
In formal concept analysis, databases are represented as formal contexts, i.e. structures $(A,X,I)$ such that $A$ and $X$ are sets, and $I \subseteq A \times X$ is a binary relation. 
Intuitively, $A$ is understood as a collection of objects, $X$ as a collection of features, and for any object $a$ and  feature $x$, the tuple $(a,x)$ belongs to $I$ exactly when object $a$ has feature $x$. Every formal context $(A,X,I)$ can be associated with the lattice of its formal concepts  (see Section \ref{ssec:DS:structures} for more details).   
FCA relies on Birkhoff's representation theory of complete lattices \cite{LatticesAndOrder} that states that any complete lattice can dually be represented by a formal context and vice-versa.

\paragraph{Lifting methodology.} 
On the basis of algebraic and logical insights, in \cite{roughconcepts} a  methodology has been introduced which generalizes Rough Set Theory from sets to  {\em formal   concepts}
 and in  \cite{GLMP18, ICLA2019paper}, the proof-theoretic aspects of this generalization are explored. 
Building on these insights and results, and in particular building on~\cite[Section 7.3]{roughconcepts}, we propose a generalization of Dempster-Shafer theory to formal concepts. 

\paragraph{Structure of the paper.}
This paper develops  preliminary theoretical results on the  generalization of Dempster-Shafer theory applied  to formal concepts and illustrates its potential at formalizing decision-making problems concerning categorization.  

In Section \ref{sec:order:th:analysis}, we provide the preliminaries on probability spaces, belief functions and plausibility functions, then we 
propose an order-theoretic analysis of the toggle between finite probability spaces and belief functions on finite sets, with a particular focus on  \cite[Theorem~3.3]{fagin1991uncertainty}.

In Section \ref{sec:conceptual:DS:structures}, we generalise the result of \cite[Theorem~3.3]{fagin1991uncertainty} from probability spaces to formal concepts. 
In Section \ref{ssec:DS:structures},
we provide preliminaries on formal concepts, introduce 
conceptual $DS$-structures and generalise the notions of belief functions and probability spaces for conceptual $DS$-structures.
In Section \ref{ssec:theorem}, we prove that belief and plausibility functions on conceptual $DS$-structures can be represented as the inner and outer measures of some probability space (see Theorem \ref{th:3.3forlat}). We  provide both a purely algebraic proof and a frame theoretic proof of that result.
The former highlights why the construction provides 
belief and plausibility functions, while the latter allows the reader to understand the structure of conceptual probability space on which the inner and outer measures are interpreted.
In Section \ref{ssec:combining}, we show how to adapt Dempster-Shafer combination of evidence to formal concepts.

In Sections \ref{sec:ex:preferences} and \ref{ssec:ex:objects}, we present two examples to illustrate how Dempster-Shafer theory can be used for formal concepts. In Section \ref{sec:ex:preferences}, we illustrate how Dempster-Shafer combination of evidence can be used to  aggregate preferences. 
In Section \ref{ssec:ex:objects}, we show how Dempster-Shafer theory can be used to make categorisation decision. Namely to answer the question: to which category does an unknown object belongs?

\section{Order-theoretic analysis}
\label{sec:order:th:analysis}

Belief and plausibility functions are one proposal among others to generalise probabilities to situations in which some predicates cannot be assigned subjective probabilities.
In this section, we collect preliminaries on belief and plausibility functions on sets (for more details on imprecise probabilities see \cite{Walley1991}), finite probability spaces, and develop an order theoretic analysis of these notions which leads to an algebraic reformulation of \cite[Theorem 3.3]{fagin1991uncertainty}.  
\paragraph{Belief, plausibility and mass functions.}
\label{def:bel-func}
A {\em belief function} (cf.~\cite[Chapter 1, page 5]{shafer1976mathematical}) on a set $S$ is a map $\bel: \mathcal{P}(S)\to [0,1]$ such that 
$\bel(S)=1$,  and for every $n\in \mathbb{N}$,
\begin{equation} 
\bel (  A_1 \cup \dots \cup A_n) \ \geq  \
\sum_{\varnothing \neq I \subseteq \{1, \dots , n\}}
(-1)^{|I|+1} \bel \left( \bigcap_{i \in I} A_i \right).
\end{equation}
A {\em plausibility function on} $S$ is a map $\pl: \mathcal{P}(S)\to [0,1]$ such that 
$\pl(S)=1$,  and for every $n\in \mathbb{N}$,
\begin{equation} 
\pl (A_1 \cup A_2 \cup ... \cup A_n) \ \leq \ 
\sum_{\varnothing \neq I \subseteq \{1,2,...,n\} }
(-1)^{|I| +1}\pl 
\left( \bigcap_{i \in I} A_i 
\right).
\end{equation}
Belief and plausibility functions on sets are interchangeable notions: for every belief function $\bel$ as above, the assignment  $X\mapsto 1- \bel(X)$ defines a plausibility function on $S$, and for every plausibility function $\pl$ as above, the assignment  $X\mapsto 1- \pl(X)$ defines a belief function on $S$.
A {\em mass function} on a set $S$ is a map $\mass: \mathcal{P}(S)\to [0,1]$ such that 
\begin{equation} 
\sum_{X \subseteq S} \mass (X) = 1.
\end{equation}
On finite sets, belief (resp.~plausibility) functions and mass functions are interchangeable notions:  any mass function $\mass$ as above induces the belief function   $\bel_{\mass}: \mathcal{P}(S)\to [0,1]$ defined as 
\begin{equation} 
\bel_\mass(X) := \sum_{Y \subseteq X} \mass(Y) \qquad \text{ for every } X \subseteq S,
\end{equation}
and conversely, any belief function $\bel$ as above induces the mass function   $\mass_{\bel}: \mathcal{P}(S)\to [0,1]$ defined as 
\begin{equation} 
\mass_\bel (X) := \bel(X) - \sum_{Y \subseteq X} (-1)^{|X \smallsetminus Y|} \bel(Y) \qquad \text{ for every } X \subseteq S.
\end{equation}

A {\em $DS$-structure} 
\footnote{
$DS$-structures are also known in the literature as tuples $(S, \bel, \pi)$ 
(cf.~\cite{fagin1991uncertainty}) where $\pi$ is a valuation. 
This valuation is necessary to show \cite[Proposition 3.4]{fagin1991uncertainty} which shows the equivalence between arbitrary and finite $DS$-structures. 
Since in this paper we restrict ourselves to finite structures the valuation here is redundant.
}
is a tuple $(S, \bel)$ with $S$ a set, $\bel : \mathcal{P}(S) \rightarrow [0,1]$. 
Central to Dempster-Shafer theory is the rule of combination of beliefs (representing e.g.~evidence, hints, or preferences) (see \cite{shafer1976mathematical}). Given two  
$DS$-structures $(S, \bel_1)$ and $(S, \bel_2)$ over a set $S$, Dempster-Shafer  rule for combining belief functions provides a procedure to compute a new 
$DS$-structures $(S, \bel_{1 \oplus 2})$ that represent the belief obtained from aggregating the information from $\bel_1$ and $\bel_2$. To  compute the aggregated belief, one combines first the mass functions $\mass_1$ and $\mass_2$ that arise from $\bel_1$ and $\bel_2$ as follows:
\begin{align}
\label{eq:combinatin rule:prelim}
\mass_{1 \oplus 2} \ : \ \mathcal{P}(S) & \rightarrow [0,1]
\\
c & \mapsto 
\left\{
    \begin{aligned}
        \; &0  & \mbox{if } \val{c}=\varnothing \\
        \; &\frac{\sum \{ \mass_1(c_1) \cdot \mass_2(c_2) \mid c_1 \cap c_2 = c  \} }{\sum \{  \mass_1(c_1) \cdot \mass_2(c_2) \mid c_1 \cap c_2 \neq \varnothing    \} }  
        & \mbox{otherwise.}
    \end{aligned}
\right.
\notag
\end{align}
Dempster-Shafer  rule allows to combine possibly conflicting beliefs proceeding from multiple independent sources. Intuitively, the combined belief  according to the Dempster-Shafer rule highlights the converging portions of evidence, and downplays the conflicting ones.

Recall that a {\em$\sigma$-algebra} on a set $S$ is a collection $\Sigma$ of subsets of $S$ that includes $S$ itself, is closed under complement, and is closed under countable unions. 
A {\em probability space} (cf.~\cite[Section 2, page 3]{fagin1991uncertainty}) is a structure $\mathbb{X} = (S, \mathbb{A}, \mu)$ where $S$ is a nonempty (finite) set, $\mathbb{A}$ is a $\sigma$-algebra on $S$, and $\mu: \mathbb{A}\to [0, 1]$ is a countably additive probability measure.  Let  $e: \mathbb{A}\hookrightarrow \mathcal{P}(S)$ denote the natural embedding of $\mathbb{A}$ into the powerset algebra of $S$. 
Any  $\mu$ as above induces the {\em inner} and {\em outer measures} $\mu_\ast, \mu^\ast: \mathcal{P}(S)\to [0,1]$ (cf.~\cite[Section 2, page 4]{fagin1991uncertainty}), respectively defined as 
\begin{align}
\mu_\ast(Z): = \text{sup}\{\mu (b)\mid b\in \mathbb{A} \text{ and } e(b)\subseteq Z\} \quad\text{ and }\quad \mu^\ast(Z): = \text{inf}\{\mu (b)\mid b\in \mathbb{A}\text{ and } Z\subseteq e(b)\}.
\end{align}
By construction, $\mu_\ast(e(b)) = \mu(b) = \mu^\ast(e(b))$ for every $b\in \mathbb{A}$ and $\mu^\ast (Z) = 1-\mu_\ast(Z)$ for every $Z\subseteq S$. 
Moreover, for every  probability space $\mathbb{X} = (S, \mathbb{A}, \mu)$, the inner (resp.~outer) measure induced by $\mu$ is a  belief (resp.~plausibility) function on $S$ (cf.~\cite[Proposition 3.1]{fagin1991uncertainty}). For more details on probability spaces, we refer the reader to \cite{MeasureTheory}.

\paragraph{Order-theoretic analysis.} In a finite probability space  $\mathbb{X}$  as above, the natural embedding $e: \mathbb{A}\hookrightarrow \mathcal{P}(S)$ is a {\em complete} lattice homomorphism (in fact, it is a complete Boolean algebra homomorphism, but in the context of Boolean algebras, these two notions collapse). Hence, the right and left adjoints of $e$ exist, denoted $\iota, \gamma:  \mathcal{P}(S)\twoheadrightarrow \mathbb{A}$ respectively, and defined as  
\begin{align}
\iota( Y) := \bigcup\{a\in \mathbb{A}\mid e(a)\subseteq Y\}\quad\text{ and }\quad 
\gamma (Y): = \bigcap\{a\in \mathbb{A}\mid Y\subseteq e(a)\}.
\end{align}
\begin{lemma}
\label{lemma: from partial probability to modal operators}
For every finite probability space $\mathbb{X} = (S, \mathbb{A}, \mu)$, and every  $Y\in \mathcal{P}(S)$,
\begin{align}
 \mu_\ast(Y) = \mu(\iota(Y))\; \mbox{  and } \;\mu^\ast(Y) = \mu(\gamma(Y)).
\end{align}  
\end{lemma}
\begin{proof}
We only show the first identity. From the definitions of $\mu_*, \iota$ and the additivity of $\mu$, we get: 
$
\mu_\ast(Y)  = \bigvee\{\mu(a)\mid a\in \mathbb{A} \text{ and } e(a)\subseteq Y\}
 = \mu (\bigcup\{a\mid a\in \mathbb{A} \text{ and } e(a)\subseteq Y\})
 = \mu (\iota( Y)).
$
\end{proof}
The next proposition is an algebraic reformulation of \cite[Theorem 3.3]{fagin1991uncertainty}.

\begin{theorem}
	For any belief function $\bel:\mathcal{P}(S)\to[0,1]$ on a finite set $S$, there exists a  finite probability space $\mathbb{X} = (S', \mathbb{A}, \mu)$ and a Boolean algebra embedding $h:\mathcal{P}(S)\hookrightarrow \mathcal{P}(S')$ such  that $\bel(X)=\mu_*(h(X))$.
\end{theorem}
\begin{proof}
Let $S':=\{(X,u)\mid X\subseteq S\ \text{ and }\ u\in X\}$, and let $\mathbb{A}$ be the Boolean subalgebra of $\mathcal{P}(S')$ generated by $\{X^*\mid X \in\mathcal{P}(S)\}$, where  $X^*:=\{(X,u)\mid u\in X\}$ for every $X\in\mathcal{P}(S)$.
	Since $S'$ is finite - as discussed above - we  can assume that $\bel$ arises from a mass function $\mass$ on $\mathcal{P}(S)$. 

	Notice that, for every $X,Y\in\mathcal{P}(S)$, if $X\neq Y$, then $X^*\cap Y^*=\varnothing$.
	Hence, the elements $\{X^*\mid X \in\mathcal{P}(S)\}$ are the atoms of $\mathbb{A}$. Therefore, we can define the probability measure $\mu: \mathbb{A}\to [0, 1]$ by taking  $\mu(X^*): =\mass(X)$ for any $X\in \mathcal{P}(S)$, and then extending it by additivity to the whole domain of  $\mathbb{A}$.

Let  $h:\mathcal{P}(S)\to \mathcal{P}(S')$ be defined by the assignment $h(X):=\{(Y,u)\in S'\mid u\in X\}$. It is routine to check that $h$ is an injective Boolean algebra homomorphism.

	Finally, let us show that $\mu_*(h(X))=\bel(X)$ for any $X\in\mathcal{P}(S)$. 
Notice that we have
\begin{align*}
\mu_*(h(X)) & = \mu ( \iota (h(X))) = 
\mu ( \bigcup\{a\in \mathbb{A}\mid e(a)\subseteq h(X)\} ) 
 =  
\mu ( \bigcup\{Y^* \mid Y \in \mathcal{P}(S) \text{ and } Y^* \subseteq h(X)\} ) .
\end{align*}
Indeed, since $\mathbb{A}$ is a finite Boolean algebra, every $a \in \mathbb{A}$ is join-generated by the atoms below it.
Since $Y^*\subseteq h(X)$ iff $Y\subseteq X$ and since  every two distinct generators are disjoint, we have  that $\mu(W)=\sum_{Y\subseteq X}\mu(Y^*)=\sum_{Y\subseteq X}\mass(Y)=\bel(X)$, as required.
	
\end{proof}

\section{Conceptual $\DS$-structures and conceptual probability spaces}
\label{sec:conceptual:DS:structures}
	
In this section, we generalise the result of \cite[Theorem~3.3]{fagin1991uncertainty} from probability spaces to formal concepts. 
In Section \ref{ssec:DS:structures},
we provide preliminaries on formal concepts, introduce 
conceptual $DS$-structures and generalise the notions of belief functions and probability spaces for conceptual $DS$-structures.
In Section \ref{ssec:theorem}, we prove that belief and plausibility functions on conceptual $DS$-structures can be represented as the inner and outer measures of some probability space. We  provide both a purely algebraic proof and a frame theoretic proof of that result.
In Section \ref{ssec:combining}, we show how to adapt Dempster-Shafer combination of evidence to formal concepts.
	
	\subsection{Preliminaries and definitions}
	\label{ssec:DS:structures}
	\paragraph{Formal contexts and their concept lattices.}
A {\em formal context} \cite{ganter2012formal}  is a structure $\mathbb{P} = (A, X, I)$ such that $A$ and $X$ are sets, and $I\subseteq A\times X$ is a binary relation. 
Formal contexts can be thought of as abstract representations of databases, where elements of $A$ and $X$ represent objects and features, respectively, and the relation $I$ records whether a given object has a given feature. 
Every formal context as above induces maps $(\cdot)^\uparrow: \mathcal{P}(A)\to \mathcal{P}(X)$ and $(\cdot)^\downarrow: \mathcal{P}(X)\to \mathcal{P}(A)$, respectively defined by the assignments 
\begin{equation}
B^\uparrow: = 
\{x\in X\mid \forall a(a\in B\Rightarrow aIx)\}\quad\text{ and }\quad 
Y^\downarrow: = 
\{a\in A\mid \forall x(x\in Y\Rightarrow aIx)\}.
\end{equation}
A {\em formal concept} of $\mathbb{P}$ is a pair 
$c = (\val{c}, \descr{c})$ such that $\val{c}\subseteq A$, $\descr{c}\subseteq X$, and $\val{c}^{\uparrow} = \descr{c}$ and $\descr{c}^{\downarrow} = \val{c}$. 
A subset $B \subseteq A$ (resp.\ $Y\subseteq X$) is said to be {\em closed} if $B=B^{\uparrow\downarrow}$ (resp.\ $Y=Y^{\downarrow\uparrow}$).
The set of objects $\val{c}$ is intuitively understood as the {\em extension} of the concept $c$, while  the set of features $ \descr{c}$ is understood as its {\em intension}. 
The set $L(\mathbb{P})$  of the formal concepts of $\mathbb{P}$ can be partially ordered as follows: for any $c, d\in L(\mathbb{P})$, 
\begin{equation}
c\leq d\quad \mbox{ iff }\quad \val{c}\subseteq \val{d} \quad \mbox{ iff }\quad \descr{d}\subseteq \descr{c}.
\end{equation}
With this order, $L(\mathbb{P})$ is a complete lattice, the {\em concept lattice} $\mathbb{P}^+$ of $\mathbb{P}$. As is well known, any complete lattice $\mathbb{L}$ is isomorphic to the concept lattice $\mathbb{P}^+$ of some formal context $\mathbb{P}$. A  formal context  $\mathbb{P}$ is {\em finite} if its associated concept lattice $\mathbb{P}^+$ is a finite lattice.\footnote{Notice that if $\mathbb{P} = (A, X, I)$ is such that $A$ and $X$ are finite sets then $\mathbb{P}^+$ is a finite lattice, but the converse is not true in general. For instance, if $\mathbb{P} = (A, X, I)$ gives rise to 
a finite lattice, then so does
$\mathbb{P}' := (A', X, I')$ where $A': = A\cup \mathbb{N}$ and $a'I' x$ iff $a'\in A$ and $aIx$.} 

\paragraph{Conceptual $\DS$-structures and conceptual probability spaces.} The notions of $\DS$-structures and  probability spaces can be generalized from sets to polarities as follows.  
\begin{definition} 
\label{def:lattice:bel-func:pl-func}
If $\mathbb{P} = (A, X, I)$ is a (finite)  formal context, a \textit{mass function} on $\mathbb{P}$ is a map $\mass  : \mathbb{P}^+  \rightarrow [0,1]$ such that 
	$\sum_{c \in \mathbb{P}^+} \mass (c) = 1$, and if $\val{\bot} = \varnothing$ then $\mass (\bot)=0$. 
	
	A \textit{conceptual $\DS$-structure} is a tuple $\mathbb{D}=(\mathbb{P},\mass)$ such that $\mathbb{P}$ is finite formal context, and $\mass$ is a  mass function on $\mathbb{P}$. 
	\end{definition}

Notice that  for a formal context $\mathbb{P} = (A, X, I)$, the bottom element of $\mathbb{P}^+$ is defined as $\bot = (\val{\bot}, \descr{\bot})$, where $\descr{\bot}: = X$ and $\val{\bot}: = \descr{\bot}^{\downarrow} = \{a\in A\mid \forall x (a I x)\}$. 
Hence, the extension of $\bot$  might in some cases be  nonempty. If so, it is implausible to require the mass of $\bot$ to always be zero. This explains the requirement that $\mass (\bot) = 0$ applies only if $\val{\bot} = \varnothing$. 

Belief and plausibility functions arise from conceptual $\DS$-structures as follows.
\begin{definition} 
\label{def:from mass to bel}
For any conceptual $\DS$-structure $\mathbb{D}=(\mathbb{P},\mass)$, let  $\bel_\mass :  \mathbb{P}^+ \rightarrow [0,1]$ and $\pl_\mass :  \mathbb{P}^+ \rightarrow [0,1]$ be defined by the following assignments:  for every $c \in \mathbb{P}^+$, 
\begin{align}
\label{eq:def:bl:pl}
\bel_\mass(c) := \sum_{c' \leq c} \mass (c')
\qquad \text{ and } \qquad  
\pl_\mass(c) := \sum_{\val{c' \wedge c} \neq \varnothing}  \mass (c').
\end{align}
\end{definition} 

\begin{definition} 
\label{def:concept:PS} A {\em conceptual probability space} is a structure $\mathbb{X} = (\mathbb{P}, \mathbb{A}, \mu)$ where $\mathbb{P}$ is a  finite formal context, $\mathbb{A}$ is a $\sigma$-{\em algebra of concepts of} $\mathbb{P}$, i.e.~a  lattice embedding $e: \mathbb{A}\hookrightarrow \mathbb{P}^+$ exists of $\mathbb{A}$ into the concept lattice of $\mathbb{P}$, and $\mu: \mathbb{A}\to [0, 1]$ is a countably additive probability measure.   \end{definition} 

In any conceptual probability space $\mathbb{X}$ as above, the  embedding $e: \mathbb{A}\hookrightarrow \mathbb{P}^+$ is a {\em complete} lattice homomorphism, and hence, similarly to the Boolean setting, the right and left adjoints of $e$ exist, denoted $\iota, \gamma:  \mathbb{P}^+\twoheadrightarrow \mathbb{A}$ respectively, and defined as  
\begin{equation}
\iota( c) := \bigvee\{a\in \mathbb{A}\mid e(a)\leq c\}\quad\text{ and }\quad \gamma (c): = \bigwedge\{a\in \mathbb{A}\mid c\leq e(a)\}.
\end{equation}
Using these maps, we can define the {\em inner} and {\em outer measures} $\mu_\ast, \mu^\ast: \mathbb{P}^+\to [0, 1]$ as follows: 
for every  $c\in \mathbb{P}^+$,
\begin{align}
\mu_\ast(c) = \mu(\iota(c))\; \mbox{  and } \;\mu^\ast(c) = \mu(\gamma(c)).
\end{align}

	\subsection{Representing conceptual $\DS$-structures as conceptual probability spaces}
	\label{ssec:theorem}
	The aim of this section is to prove Theorem \ref{th:3.3forlat}. We  provide both a purely algebraic proof and a frame theoretic proof of that result.
The former highlights why the construction provides 
belief and plausibility functions, while the latter allows the reader to understand the structure of conceptual probability space on which the inner and outer measures are interpreted.

\begin{theorem}\label{th:3.3forlat}
	For any conceptual DS-structure $\mathbb{D}=(\mathbb{P},\mass)$ there exists a  finite conceptual probability space $\mathbb{X} = (\mathbb{P}', \mathbb{A}, \mu)$, and a meet-preserving embedding $h:{\mathbb{P}}^+\hookrightarrow \mathbb{P}'^+$ such  that, for every $c \in \mathbb{P}^+$, we have $\bel_\mass(c)=\mu_*(h(c))$ and $\pl_\mass(c)=\mu^*(h(c))$.
\end{theorem}

\noindent\textbf{Algebraic proof of Theorem \ref{th:3.3forlat}.}
Let $\mathbb{D}=(\mathbb{P},\mass)$ be a finite conceptual DS-structure. 
Let $\mathbb{P}^+$ be the lattice of concepts associated to the formal context $\mathbb{P}$. Notice that  $\mathbb{P}^+$ is  finite. By definition, we have $\mass:\mathbb{P}^+\to[0,1]$ such that $\sum_{a\in\mathbb{P}^+}\mass(a)=1$. 

For every $a\in\mathbb{P}^+$, let $\mathbb{L}_a:=\{b\in\mathbb{P}^+\mid b\leq a\}$.  $\mathbb{L}_a$ is a sublattice of $\mathbb{P}^+$.
For every $b\in\mathbb{P}^+$, let $b_a\in\mathbb{L}_a$ be the element such that $b_a=b\land a$.
Let $\mathbb{L}':=\prod_{a\in \mathbb{A}}\mathbb{A}_a$ be the product of these lattices.
For every $b \in \mathbb{P}^+$, let $b^* \in \mathbb{L}'$ be defined as follows, for every $a \in \mathbb{P}^+$,
\begin{equation}
b^*(a) = \left\{
\begin{array}{ll}
b & \text{if } b=a,\\
\bot & \text{otherwise.}
\end{array}
\right.
\end{equation}

Let us show that the sublattice $\mathbb{A}$ generated by $\{b^*\in\mathbb{L}'\mid b\in\mathbb{P}^+\}$ is a Boolean algebra.
Notice that $\{b^*\in\mathbb{L}'\mid b\in\mathbb{P}^+\}$ are exactly the atoms of $\mathbb{A}$. 
Hence, every element  $c\in\mathbb{A}$ can be written as $\bigvee_{b\in S}b^*$ for some $S\subseteq \mathbb{P}^+$. 
Define $\lnot c:=\bigvee_{b\in\mathbb{P}^+\smallsetminus S}b^*$. 
It's immediate that $\lnot$ is a Boolean negation. This completes the proof that $\mathbb{A}$  is a Boolean algebra. Let $\mathbb{P}'$ be a finite formal context  associated to $\mathbb{L}'$. Hence, $\mathbb{A}$ is trivially a $\sigma$-algebra of concepts of $\mathbb{P}'$.
Let $\mu  :  \mathbb{A} \rightarrow [0,1]$ be defined as follows. For every $b \in \mathbb{P}^+$,  we take $\mu(b^*) := \mass(b)$. Then, we extend the map $\mu$ using the fact that every $a \in \mathbb{A}$ is the join of the $b^*$ below it and the fact that $\mu$ is additive. Then $\mathbb{X} := (\mathbb{P}', \mathbb{A}, \mu) $ is a finite conceptual probability space.
The embedding $e  : \mathbb{A} \rightarrow \mathbb{L}'$ is the natural embedding that sends each element in $\mathbb{A}$ to its corresponding element in $\mathbb{L}'$.


Let $h:\mathbb{P}^+\to\mathbb{L}'$ be such that 
\begin{align}
h(b)(a)=b_a = b \wedge a.
\end{align}
%
For every $a,c,d \in \mathbb{P}^+$, we have 
\begin{align*}
[h(c)\wedge h(d)](a) = 
c_a\land d_a=
(c\land a)\land (d\land a) = (c\land d)\land a=h(c \wedge d)(a).
\end{align*}
Hence, $h(c)\wedge h(d) = h(c \wedge d)$ and $h$ is meet-preserving.\\


Let us show that $\bel(c)=\mu_*(h(c))$ for every $c \in \mathbb{P}^+$. Since $\mathbb{A}$ is an atomic Boolean algebra, every element is equal to the join of the atoms below it.
Furthermore, notice that $d\leq c$ if and only if $d^*\leq h(c)$. Therefore, we have:
\begin{align*}
\mu_*(h(c)) & = \mu( \iota(h(c))) =
\mu \left(  
\bigvee \left\{a \in \mathbb{A} \mid e(a) \leq h(c)
\right\}
\right)
\\ 
& = \mu \left(  
\bigvee \left\{d^* \in \mathbb{A} \mid d\in \mathbb{P}^+ \; \text{ and } \; d^* \leq h(c)
\right\}
\right)
\tag{because each $a \in \mathbb{A}$ is equal to the join of the atoms below it}
\\
& = \mu \left(  
\bigvee \left\{d^* \in \mathbb{A} \mid d \leq c
\right\}
\right)
\\
& = \sum \{ \mu(d^* ) \mid  d \leq c \}
\tag{because $\mu$ is additive}
\\
& = \sum \{ \mass(d) \mid  d \leq c \}
\tag{by definition of $\mu$}\\
& = \bel_\mass(c).
\tag{by definition of $\bel$}
\end{align*}

Finally, let us show that  $\pl(c)=\mu^*(h(c))$ for every $c \in \mathbb{P}^+$. Let $U\subseteq{\mathbb{P}}^+$ and $c \in \mathbb{P}^+$. 
Notice that if $h(c)(d)\neq\bot$, then 
$h(c)(e) = c \wedge e \leq \bigvee_{d\in U}d^*(e)$ if and only if $e\in U$. Indeed, $\bigvee_{d\in U}d^*(e) = e$ if $e \in U$ and $\bot$ otherwise. 
We have:
\begin{align}
\bigwedge \left\{ \bigvee_{d\in U}d^* \; \middle| \; U\subseteq{\mathbb{P}}\ \text{ and } \ h(c)\leq \bigvee_{d\in U}d^*\right\}
& = \bigwedge \left\{
\bigvee_{d\in U}d^* \; \middle| \ U\subseteq{\mathbb{P}}\ \text{ and } \   \; d\land c\neq\bot \; \text{ implies } \; d\in U \right\}
\notag
\\
& =\bigvee\{d^*\mid d\land c\neq\bot\}.
\label{align:proof:alg:th:pl}
\end{align}  
The last equality holds because $\mathbb{A}$ is a Boolean algebra. Hence, we have:
\begin{align*}
\mu^*(h(c)) & = \mu(\gamma(h(c)))  
= \mu \left(  \bigwedge\{a\in \mathbb{A} \mid h(c)\leq e(a)\} \right)
\\
& =  \mu \left(\bigwedge
\left\{\bigvee_{d\in U}d^* \; \middle| \; U\subseteq{\mathbb{P}}\ \text{ and }\ h(c)\leq \bigvee_{d\in U}d^*\right\}\right)
\tag{because each $a \in \mathbb{A}$ is equal to the join of the atoms $d^*$ below it}
\\
& = \mu \left( \bigvee \{ d^* \mid c \land d \neq \bot \} \right)
\tag{see \eqref{align:proof:alg:th:pl}}
\\
& = \sum \{ \mu ( d^* ) \mid c \land {d} \neq \bot \} 
\tag{because $\mu$ is additive}
\\
& = \sum \{ \mass (d) \mid {c} \land {d} \neq \bot \} 
\tag{by definition of $\mu$}
\\
& = \pl_\mass (c).
 \tag{by definition of $\pl$}
\end{align*}

\medskip

\noindent\textbf{Frame theoretic proof of Theorem \ref{th:3.3forlat}.}
We proceed via a series of lemmas.
Let $\mathbb{P}=(A,X,I)$ be a finite formal context and ${\mathbb{P}}^+$ be its lattice of concepts. We can assume without loss of generality that $X^{\downarrow}=\varnothing$. We define the formal context $\mathbb{P}'=(A',X',I')$ as follows: 
\begin{itemize}
	\item $A' := \{(c,a) \mid c \in \mathbb{P}^+ \text{ and }  a\in \val{c} \}$;
	\item $X' := \{(c,x) \mid   c \in \mathbb{P}^+ \text{ and } x\in X \}$;
	\item $(c,a)I'(d,x)$ if and only if $c\neq d$   or    $ aIx$.
\end{itemize}
For every $c\in {\mathbb{P}}^+$, let $c^*=(\{(c,a)\mid a\in \val{c}\},\{(c,a)\mid a\in \val{c}\}^{\uparrow})$. 
\begin{lemma}
For every $c\in {\mathbb{P}}^+$, the set $\{(c,a)\mid a\in \val{c}\}$ is the extension of a formal concept.
\end{lemma} 
\begin{proof}
	Notice that $\{(c,a)\mid a\in \val{c}\}^{\uparrow}=\{(d,x)\in X'\mid d\neq c\}\cup\{(c,x)\in X'\mid x\in\descr{c}\}$. Since, by assumption, $X^{\downarrow}=\varnothing$, it follows that $\{(d,x)\in X'\mid d\neq c\}^{\downarrow}=\{(c,a)\in A'\mid a\in\val{c}\}$. Since the map $(\cdot)^{\downarrow}$ is antitone, $\{(c,a)\mid a\in \val{c}\}^{\uparrow\downarrow}\leq \{(d,x)\in X'\mid d\neq c\}^{\downarrow}=\{(c,a)\in A'\mid a\in\val{c}\}$.
\end{proof}
The lemma above  implies that  $c^* \in \mathbb{P'}^+$. Let $\mathbb{A}$ be the sub join-semilattice of $\mathbb{P'}^+$ join-generated by $\{c^*\mid c\in{\mathbb{P}}^+\}$. 

\begin{lemma}
The lattice $\mathbb{A}$ is a finite Boolean algebra, generated by the set of its atoms $\{c^*\mid c\in{\mathbb{P}}^+\}$
\end{lemma}
\begin{proof}
That the generators are atoms follows immediately from the fact that if $c\neq d$, then $\val{c^*}\cap\val{d^*}=\emptyset$.
To show that $\mathbb{A}$ is a Boolean algebra, it is enough to show that for any $U\subseteq{\mathbb{P}}^+$, the set $\bigcup_{c\in U}\val{c^*}$ is a closed subset of $A'$. We have $(\bigcup_{c\in U}\val{c^*})^{\uparrow}=\bigcap_{c\in U}\descr{c^*}$. As discussed in the lemma above, $\descr{c^*}=\{(d,x)\in X'\mid d\neq c\}\cup\{(c,x)\in X'\mid x\in\descr{c}\}$. Therefore $\bigcap_{c\in U}\descr{c^*}\supseteq\{(d,x)\in X'\mid d\notin U \}$. Since $X^{\downarrow}=\varnothing$, it follows that  $\{(d,x)\in X'\mid d\notin U \}^{\downarrow}=\{(c,a)\in A'\mid c\in U\}=\bigcup_{c\in U}\val{c^*}$. Therefore $(\bigcup_{c\in U}\val{c^*})^{\uparrow\downarrow} = (\bigcap_{c\in U}\descr{c^*})^{\downarrow}\subseteq \bigcup_{c\in U}\val{c^*}$, as required.
\end{proof}

\begin{remark} 
The reader familiar with the proof of \cite[Theorem~3.3]{fagin1991uncertainty} will notice that the  algebra $\mathbb{P}'^+$ is  comparatively as large as the one used to prove \cite[Theorem~3.3]{fagin1991uncertainty}. Had we chosen to define  $\mathbb{P}'^+$ in a simpler way, e.g.~as the power set Boolean algebra over the set of concepts $\mathcal{P}(P^+)$, we would have run into problems because $\mathcal{P}(P^+)$ would have forced inner and outer measures, hence belief and plausibility functions, to coincide. 
Indeed, in principle the $\sigma$-algebra with an arbitrary measure arising from a mass function on $P^+$ needs to have size at least $\mathcal{P}(P^+)$. 
%
\end{remark}

  Since $\mathbb{A}$ is a Boolean algebra, we can  define the map $\mu  :  \mathbb{A} \rightarrow [0,1]$ 
first on the generators of $\mathbb{A}$, by letting
$\mu(c^*)=\mass(c)$ for every $ c\in \mathbb{P}^+$, and  uniquely extend it to a measure on $\mathbb{A}$. Let $e: \mathbb{A}\hookrightarrow \mathbb{P}'^+ $ be the natural embedding. By construction, $e$ is a lattice homomorphism, hence the right and left adjoint of $e$ exist, denoted   $\iota$ and $\gamma$, respectively. 
%
%

Let us define the map $h:{\mathbb{P}}^+\hookrightarrow \mathbb{P}'^+$ as follows: 
for every  $c \in \mathbb{P}^+$,
\begin{equation}
h(c)=(\{(d,a)\in A'\mid a\in\val{c}\},\{(d,a)\in A'\mid a\in\val{c}\}^{\uparrow}).
\end{equation}
In the two lemmas below we show that $h$ is indeed well defined:
\begin{lemma} 
	$\{ (d,a) \mid a \in  \val{c}\}^{\uparrow}=\{(e,x)\in X'\mid x\in(\val{e}\cap\val{c})^{\uparrow}\}$.
\end{lemma}
\begin{proof}
	Left to right inclusion: Assume $(e,x)\in \{ (d,a) \mid a \in \val{c}\}^{\uparrow}$. If $e\neq d$ for all $(d,a)\in  \{ (d,a) \mid a \in  \val{c}\}$, then $\val{e}\cap\val{c}=\emptyset$, hence $x\in(\varnothing)^{\uparrow}$. If $d=e$, then $aIx$ for all $a\in \val{e}\cap\val{c}$, hence $x\in (\val{e}\cap\val{c})^{\uparrow}$. 
	
	Right to left inclusion: Let $(e,x)\in \{(e,x)\in X'\mid x\in(\val{e}\cap\val{c})^{\uparrow}\}$ and let $(e,a)\in  \{ (d,a) \mid a \in \val{c}\}$. Since $x\in (\val{e}\cap\val{c})^{\uparrow}$ by assumption, it follows that $aIx$.
\end{proof}

\begin{lemma}
	$\{(e,x)\in X'\mid x\in(\val{e}\cap\val{c})^{\uparrow}\}^{\downarrow}\subseteq\{(d,a)\in A'\mid x\in \val{c}\}$
\end{lemma}
\begin{proof}
	Let $(d,a)$ be such that $a\notin\val{c}$. Then $a\notin\val{c}\cap\val{d}$ and there exists $x_0\in(\val{d}\cap\val{c})^{\uparrow}$ such that $(a, x_0)\notin I$. Hence, $(d,x_0)\in   \{(e,x)\in X'\mid x\in(\val{e}\cap\val{c})^{\uparrow}\}$ and $(d,a)\notin \{(e,x)\in X'\mid x\in(\val{e}\cap\val{c})^{\uparrow}\}^{\downarrow}$.
\end{proof}

\begin{lemma}
The map $h$ defined above is a meet-preserving embedding.	
\end{lemma}
\begin{proof}
Let $c,d \in \mathbb{P}^+$. If $c\neq d$, it is immediate that $h(c)\neq h(d)$. Now, let us show that $h$ is meet-preserving:
	 \begin{align*}
	 \val{h(c)}\cap\val{h(d)}=&\{(e,a)\mid x\in \val{c}\}\cap \{(e,a)\in\mid x\in \val{c}\}\\
	 =&\{(e,a)\mid x\in \val{c}\cap\val{d}\}
	 =\{(e,a)\mid x\in \val{c\land d}\}
	 =\val{h(c\land d)}.
	 \end{align*}
\end{proof}
Let us show that $\bel(c)=\mu_*(h(c))$ for every $c \in \mathbb{P}^+$. Since $\mathbb{A}$ is an atomic Boolean algebra, every element is equal to the join of the atoms below it.
Furthermore, notice that $d\leq c$ if and only if $d^*\leq h(c)$. Therefore, we have:
\begin{align*}
\mu_*(h(c)) & = \mu( \iota(h(c))) =
\mu \left(  
\bigvee \left\{a \in \mathbb{A} \mid a \leq h(c)
\right\}
\right)
\\ 
& = \mu \left(  
\bigvee \left\{d^* \in \mathbb{A} \mid d\in \mathbb{P}^+ \; \text{ and } \; d^* \leq h(c)
\right\}
\right)
\tag{because each $a \in \mathbb{A}$ is equal to the join of the atoms below it}
\\
& = \mu \left(  
\bigvee \left\{d^* \in \mathbb{A} \mid d \leq c
\right\}
\right)
= \Sigma \{ \mu(d^* ) \mid  d \leq c \}
= \Sigma \{ \mass(c) \mid  d \leq c \}
= \bel_\mass(V(p)).
\end{align*}

Finally, let us show that  $\pl(c)=\mu^*(h(c))$ for every $c \in \mathbb{P}^+$.  Let $U\subseteq{\mathbb{P}}^+$, $c \in \mathbb{P}^+$ and $(d,a) \in A'$. If $(d,a)\in h(c)$, then $(d,a)\in \bigvee_{d\in U}d^*$ if and only if $d\in U$. Furthermore, $(d,a)\in h(c)$ if and only if $a\in\val{c}\cap\val{d}$. Hence, $(d,a)\in h(c)$ implies that $\val{c}\cap\val{d}\neq\emptyset$. We have:

\begin{align}
\bigwedge \left\{ \bigvee_{d\in U}d^* \; \middle| \; U\subseteq{\mathbb{P}}\ \text{ and } \ h(c)\leq \bigvee_{d\in U}d^*\right\}
& = \bigwedge \left\{
\bigvee_{d\in U}d^* \; \middle|  \; \val{d}\cap\val{c}\neq\emptyset \; \text{ implies } \; d\in U \right\}
\notag
\\
& =\bigvee\{d^*\mid \val{d}\cap\val{c}\neq\emptyset\}.
\label{align:proof:th:pl}
\end{align}  
The last equality holds because $\mathbb{A}$ is a Boolean algebra. Hence, we have:
$\bigwedge\{a\in \mathbb{A}\mid c\leq e(a)\}$
\begin{align*}
\mu^*(h(c)) & = \mu(\gamma(h(c)))  
= \mu \left(  \bigwedge\{a\in \mathbb{A} \mid h(c)\leq e(a)\} \right)
\\
& =  \mu \left(\bigwedge
\left\{\bigvee_{d\in U}d^* \; \middle| \; U\subseteq{\mathbb{P}}\ \text{ and }\ h(c)\leq \bigvee_{d\in U}d^*\right\}\right)
\tag{because each $a \in \mathbb{A}$ is equal to the join of the atoms $d^*$ below it}
\\
& = \mu \left( \bigvee \{ d^* \mid \val{c} \cap \val{d} \neq \emptyset \} \right)
\tag{see \eqref{align:proof:th:pl}}
\\
& = \Sigma \{ \mu ( d^* ) \mid \val{c} \cap \val{d} \neq \emptyset \} 
\tag{because $\mu$ is additive}
\\
& = \Sigma \{ \mass (d) \mid \val{c} \cap \val{d} \neq \emptyset \} 
 = \pl_\mass (c).
\end{align*}

This concludes the proof of Theorem \ref{th:3.3forlat}.

	\subsection{Combination of evidence}
	\label{ssec:combining}

Here, we propose a generalized version of Dempster-Shafer rule for concepts.
Let $\mathbb{P}=(A,X,I) $ be a finite formal context and $\mass_1$, $\mass_2$ be two  mass functions on $\mathbb{P}$.
We build 
the combined  mass function $\mass_{1 \oplus 2}$ following Dempster-Shafer's procedure. For all $c_1, c_2 \in \mathbb{P}^+ $, if $\sum \{  \mass_1(c_1) \cdot \mass_2(c_2) \mid \val{c_1 \wedge c_2} \neq \varnothing    \}
\neq 0$, then the combined  mass function $\mass_{1 \oplus 2}$ is defined as follows:
\begin{align}
\label{eq:combinatin rule}
\mass_{1 \oplus 2} \ : \ \mathbb{P}^+ & \rightarrow [0,1]
\\
c & \mapsto 
\left\{
    \begin{aligned}
        \; &0  & \mbox{if } \val{c}=\varnothing \\
        \; &\frac{\sum \{ \mass_1(c_1) \cdot \mass_2(c_2) \mid c_1 \wedge c_2 = c  \} }{\sum \{  \mass_1(c_1) \cdot \mass_2(c_2) \mid \val{c_1 \wedge c_2} \neq \varnothing    \} }  
        & \mbox{otherwise.}
    \end{aligned}
\right.
\notag
\end{align}
It is straightforward to verify that $\mass_{1\oplus 2}$  is a  mass function on $\mathbb{P}$.

%


%

\section{Example: preference aggregation}
	\label{sec:ex:preferences}


In this section, we illustrate how the mass functions can be used to encode preferences and how Dempster-Shafer combination of evidence can be used to  aggregate them. 
In Section \ref{ssec:ex:objects}, we show how Dempster-Shafer theory can be used to make categorisation decision. Namely to answer the question: to which category does an unknown object belong?

\paragraph{The scenario.}
Alice and Bob wish to watch a movie together, and query a movie database by expressing their independent (graded) preferences. The software interface  interprets their preferences as mass functions on the database (modelled as a formal context), and combines them using the rule of Section \ref{ssec:combining}. \\

\noindent\textbf{Case 1. Conflict with no resolution.} Alice wishes to watch a romantic comedy and Bob a chainsaw horror movie. The database they are querying, and its associated concept lattice, look as follows:
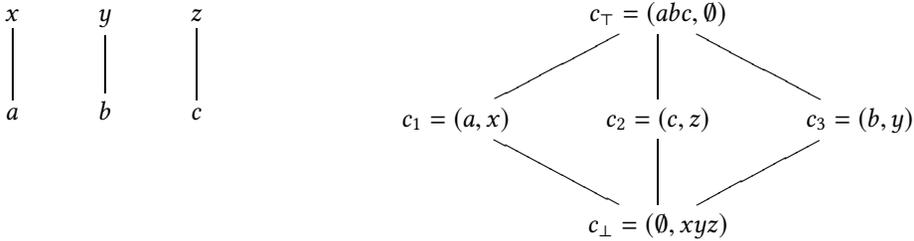
\begin{figure}[H]
\begin{center}
\begin{tabular}{cc}
\xymatrix{
x & y & z
\\
a \ar @{-} [u] 
& b  \ar @{-} [u] 
& c  \ar @{-} [u] 
}
& \hspace{2cm}
\xymatrix{
& 
c_\top = (abc,\emptyset)
&
\\
c_1 = (a,x)
\ar @{-} [ur]
& c_2 = (c,z)
\ar @{-} [u]
& c_3 = (b,y)
\ar @{-} [ul]
\\
& 
c_\bot = (\emptyset ,xyz) 
\ar @{-} [u] \ar @{-} [ur] \ar @{-} [ul]
&
}
\end{tabular}
\end{center}
\end{figure}
The diagram on the left represents the database with movies $a$, $b$ and $c$ and features $x$, $y$ and $z$. The line between $a$ and $x$ means that the movie $a$ has feature $f$. The lattice on the right represents the lattice of concepts associated to the database. 

The system interprets the categories `romantic comedy' and `chainsaw horror movie' as the formal concepts $c_1$ and $c_3$, respectively, and encodes Alice and Bob's preferences  as the following mass functions:

\begin{figure}[H]
  \begin{center}
    \begin{tabular}{|c||c|c|c|c|c|}
\hline
& $c_\bot$ & $c_1$ & $c_2$ & $c_3$ &  $c_\top$
\\
\hline
\hline
$\mass_1$ & $0$ & $0.9$  & $0$  & $0$ &  $0.1$
\\
\hline
$\mass_2$ & $0$ & $0$  & $0$  & $0.9$ &  $0.1$
\\
\hline
\end{tabular}
  \end{center}
  \label{fig:mass1}
\end{figure}
The  combination of the two masses above indicates that there is no way to accommodate both preferences in the database:
\begin{figure}[H]
  \begin{center}
    \begin{tabular}{|c||c|c|c|c|c|}
\hline
& $c_\bot$ & $c_1$ & $c_2$ & $c_3$ &  $c_\top$
\\
\hline
\hline
$\mass_{1 \oplus 2}$ & $0$ & $0.47$  & $0$  & $0.47$ &  $0.06$
\\
\hline
\end{tabular}
  \end{center}
\end{figure}

%

\noindent\textbf{Case 2.  Reaching a compromise.} Alice wishes to watch a romantic comedy but  would consider an action movie, while  Bob much prefers a chainsaw horror movie but would consider an  action movie. The database they are querying, and its associated concept lattice, are as in the case above. 
The system interprets the categories `romantic comedy', `action movie' and  `chainsaw horror movie' as the formal concepts $c_1$, $c_2$ and $c_3$, respectively, and encodes Alice and Bob's preferences  as the following mass functions:

\begin{figure}[H]
\begin{center}
\begin{tabular}{|c||c|c|c|c|c|}
\hline
& $c_\bot$ & $c_1$ & $c_2$ & $c_3$ &  $c_\top$
\\
\hline
\hline
$\mass_1$ & $0$ & $0.9$  & $0.1$  & $0$ &  $0$
\\
\hline
$\mass_2$ & $0$ & $0$  & $0.1$  & $0.9$ &  $0$
\\
\hline
\end{tabular}
\end{center}
\end{figure}

The  combination of the two masses above disregards the conflict and highlights the convergence, as in the original Dempster-Shafer setting:

\begin{figure}[H]
\begin{center}
\begin{tabular}{|c||c|c|c|c|c|}
\hline
& $c_\bot$ & $c_1$ & $c_2$ & $c_3$ &  $c_\top$
\\
\hline
\hline
$\mass_{1 \oplus 2}$ & $0$ & $0$  & $1$  & $0$ &  $0$\\
\hline
\end{tabular}
\end{center}
\end{figure}


\noindent\textbf{Case 3.  Solution to the conflict.} Alice wishes to watch a romantic comedy and Bob an action movie. This time, the database they are querying, and its associated concept lattice, look as follows:
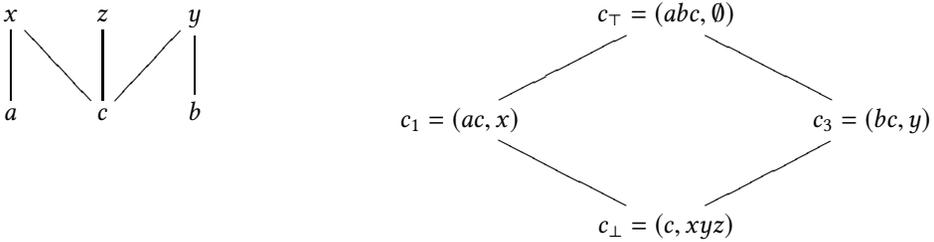
\begin{figure}[H]
\begin{center}
\begin{tabular}{cc}
\xymatrix{
x  &  z & y
\\
a \ar @{-} [u]  
& c \ar @{-} [ur] \ar @{-} [ul]  \ar @{-} [u]
& b  \ar @{-} [u]  
}
& \hspace{2cm}
\xymatrix{
& c_\top = (abc, \emptyset) &
\\
c_1 = (ac,x)
\ar @{-} [ur]
&& c_3 = (bc,y)
\ar @{-} [ul]
\\
& c_\bot = (c, xyz )
\ar @{-} [ur] \ar @{-} [ul]
&
}
\end{tabular}
\end{center}
\end{figure}
That is, this database contains and object $c$ that happens to be both a romantic comedy and an action movie. The system interprets the categories `romantic comedy', and `action movie' as the formal concepts $c_1$ and $c_3$, respectively, and encodes Alice and Bob's preferences  as the following mass functions:
\begin{figure}[H]
\begin{center}
\begin{tabular}{|c||c|c|c|c|}
\hline
& $c_\bot$ & $c_1$ & $c_3$& $c_\top$
\\
\hline
\hline
$\mass_1$ & $0$ & $0.9$  & $0$  & $0.1$
\\
\hline
$\mass_2$ & $0$ & $0$  & $0.9$  & $0.1$
\\
\hline

\end{tabular}
\end{center}
\end{figure}
The combination of the two masses points at the solution that simultaneously satisfies Alice and Bob's preferences:
\begin{figure}[H]
\begin{center}
\begin{tabular}{|c||c|c|c|c|}
\hline
& $c_\bot$ & $c_1$ & $c_3$& $c_\top$
\\
\hline
\hline
$\mass_{1 \oplus 2}$ & $0.81$ & $0.09$ & $0.09$  & $0.01$  
\\
\hline
\end{tabular}
\end{center}
\end{figure}


\section{Example:  categorization theory}
	\label{ssec:ex:objects}

 In the present section, we discuss how  the basic tools of Dempster-Shafer theory on formal contexts and their associated concept lattices introduced  in the previous sections can be applied  to the formalization of categorization decisions in various areas. The application we propose here builds on \cite{conradie2016categories,Tarkpaper}, where  formal contexts are  regarded as abstract representations of databases (e.g.~of market products and their relevant features) and their associated formal concepts  as categories, each admitting both an extensional characterization (in terms of the objects that are members of the given category) and an intensional characterization (in terms of the features that are part of the description of the given category). 

\paragraph{Music databases.}
Consider the problem of categorizing a given song $S$ on the base of user-inputs. The categorization procedure consists in aggregating the replies of users to a questionnaire about $S$. 
The questionnaire  makes reference to  the objects of a database, which, for the sake of simplicity, we represent as the following formal context
$\mathbb{P} = (A, X, I)$,
with
$A = \{a,b,c \}$  a given set of songs ($a$ = A-ha -- {\em Take on me}, $b$ = Beyonce -- {\em Crazy in love}, $c$ = Marvin Gaye -- {\em Sexual healing})
and
$X = \{w, x, y, z\}$  a set of relevant features: $w$ = keyboards, $x$ = upbeat tempo,  $y$ = gospel-trained singers, $z$= whispering voices.
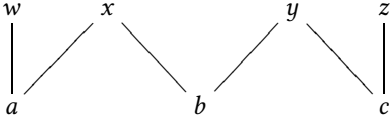
\begin{center}
\xymatrix{
w & x &&  y & z 
\\
a  \ar @{-} [u] \ar @{-} [ur]
& & b  \ar @{-} [ul] \ar @{-} [ur]
& & c  \ar @{-} [ul] \ar @{-} [u]
}
\end{center}

The questionnaire makes also reference to the (names of) the categories-as-formal-concepts arising from $\mathbb{P}$, i.e.~the elements of the following lattice $\mathbb{C}$:
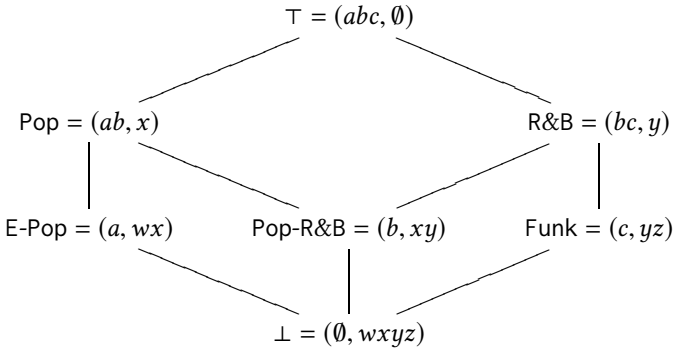
\begin{center}
\xymatrix{
& \top = (abc, \emptyset) &
\\
 \mathtt{Pop}=(ab,x)  \ar @{-} [ur] &
& \mathtt{R\&B}=(bc,y)  \ar @{-}[ul]
&
\\
\mathtt{E}\mbox{-}\mathtt{Pop}=(a,wx) \ar@{-} [u]
& \mathtt{Pop}\mbox{-}\mathtt{R\&B}=(b,xy) \ar@{-}[ul] \ar@{-}[ur]
& \mathtt{Funk}=(c,yz) \ar@{-}[u]
& 
\\
& 
\bot = (\emptyset, wxyz)
\ar@{-}[ul] \ar @{-}[ur] \ar@{-}[u] 
&&
}
\end{center}

Specifically, the questionnaire contains statements of the following  types:
\begin{itemize}
\item ``Song $S$ is similar to  song $e$'', for $e\in A$;
\item ``Song $S$ belongs to category \texttt{C}'', for \texttt{C} $\in \mathbb{C}$.
\end{itemize}
Each user chooses one or more such statements and  grades  it on  the following scale: 
$$\{0,\ 0.1,\ 0.2,\ 0.3,\ 0.4,\ 0.5,\ 0.6,\ 0.7,\ 0.8,\ 0.9,\ 1 \}.$$

Three users give the following responses: 
\begin{itemize}
\item User 1: ``Song $S$  is similar to $a$'', graded $0.2$.
\item User 2: ``Song $S$ is similar to  $a$'' graded 0.6,  and ``Song $S$ is similar to  $b$'' graded 0.6, and  ``Song $S$ belongs to  \texttt{E-Pop}'' graded 0, and ``Song $S$ belongs to  \texttt{Pop-R\&B}'' graded 0.
\item User 3: ``Song $S$ belongs to  \texttt{Pop-R\&B}" graded $0.2$ and ``Song $S$ belongs to  \texttt{Funk}" graded $0.6$ and ``Song $S$  belongs to category $\bot$'' graded $0$.
\end{itemize}

\paragraph{The pieces of evidence.}
The users' responses constitute three pieces of evidence that we will represent as mass functions  $\mass \ : \ \mathbb{C} \rightarrow [0,1]$, to each of which we can then associate belief and  plausibility functions as follows: for any $ c \in \mathbb{C}$,
$$\bel_\mass(c)=  \Sigma \{ \mass(d) \mid d \leq  c \} \qquad
\text{and} 
\qquad 
\pl_\mass(c)=  \Sigma \{ \mass(d) \mid c \wedge  d \neq \bot \}.$$

Notice that, in order for the evidence to represent a mass and not a belief, it needs to specify the probability that an object is in a category without being in any of its subcategories.


\begin{itemize}
\item 
To model User 1's response, we need a mass function that assigns mass $0.2$ to the smallest category containing the song $a$, that is, the category that most accurately describes the song $a$. This category is the category the extension of which is the closure of the singleton $\{a\}$. In our example, this  category is $\mathtt{E}\mbox{-}\mathtt{Pop}=(a,wx)$.
Hence, this statement
translates into the mass function $\mass_1 \ : \ \mathbb{C} \rightarrow [0,1]$ that maps  $\mathtt{E}\mbox{-}\mathtt{Pop}$ to $0.2$,  the category $\top$ of all songs in the database to $0.8$ and all other categories to   $0$.

\item Similarly, User $2$'s response 
translates into  the mass function $\mass_2 \ : \ \mathbb{C} \rightarrow [0,1]$ that maps  the smallest category containing both objects $a$ and $b$ (which is the category $\mathtt{Pop} = (ab,x)$) to $0.6$, the category $\top$ to $0.4$ and the remaining categories to $0$.

\item Finally, User $3$'s response 
translates into  the mass function $\mass_3 \ : \ \mathbb{C} \rightarrow [0,1]$ that maps   $\mathtt{Pop}\mbox{-}\mathtt{R\&B}$ to $0.2$, the category $\mathtt{Funk}$ to $0.6$, the category $\top$ to $0.2$, and  the remaining categories to $0$. Notice that $\bot$ might not be the empty category in general; indeed, by definition, the members of $\bot$ are the objects that share all the features in $X$.
\end{itemize}
The following table reports the values of the mass functions.
\begin{center}
\begin{tabular}{|c||c|c|c|c|c|c|c|}
\hline
& $\bot$ & $\mathtt{E}$-$\mathtt{Pop}$ & $\mathtt{Pop}$-$\mathtt{R\&B}$ & $\mathtt{Funk}$ & $\mathtt{Pop}$ & $\mathtt{R\&B}$ & $\top$
\\
\hline
\hline
$\mass_1$ & $0$ & $0.2$ & $0$ & $0$ & $0$ & $0$ & $0.8$
\\
\hline
$\mass_2$ & $0$ & $0$ & $0$ & $0$ & $0.6$ & $0$ & $0.4$
\\
\hline
$\mass_3$ & $0$ & $0$ & $0.2$ & $0.6$ & $0$ & $0$ & $0.2$
\\
\hline
\end{tabular}
\end{center}
The following table reports the belief and  plausibility functions corresponding to the mass functions above:
\begin{center}
\begin{tabular}{|c||c|c|c|c|c|c|c|}
\hline
& $\bot$ & $\mathtt{E}$-$\mathtt{Pop}$ & $\mathtt{Pop}$-$\mathtt{R\&B}$ & $\mathtt{Funk}$ & $\mathtt{Pop}$ & $\mathtt{R\&B}$ & $\top$

\\
\hline
\hline
$\bel_{\mass_1}$ & $0$  &  $0.2$ & $0$ & $0$ & $0.2$ & $0$  & $1$
\\
\hline
$\pl_{\mass_1}$ & $0$  &  $1$ & $0.8$ & $0.8$ & $1$ & $0.8$  & $1$
\\
\hline
\hline
$\bel_{\mass_2}$ & $0$  &  $0$ & $0$ & $0$ & $0.6$ & $0$ & $1$
\\
\hline
$\pl_{\mass_2}$ & $0$  &  $1$ & $1$ & $0.4$ & $0.6$ & $1$ & $1$
\\
\hline
\hline
$\bel_{\mass_3}$ & $0$  &  $0$ & $0.2$ & $0.6$ & $0.2$ & $0.8$ & $1$
\\
\hline
$\pl_{\mass_3}$ & $0$  &  $0.2$ & $0.4$ & $0.8$ & $0.4$ & $1$ & $1$
\\
\hline
\end{tabular}
\end{center}

\paragraph{Combining pieces of evidence.}
Using the combination rule  \eqref{eq:combinatin rule}, 
we get the following aggregated  mass function,  and its associated belief and plausibility functions:
\begin{center}
\begin{tabular}{|c||c|c|c|c|c|c|c|}
\hline
& $\bot$ & $\mathtt{E}$-$\mathtt{Pop}$ & $\mathtt{Pop}$-$\mathtt{R\&B}$ & $\mathtt{Funk}$ & $\mathtt{Pop}$ & $\mathtt{R\&B}$ & $\top$
\\
\hline
\hline
$\mass_{1\oplus 2 \oplus 3}$ & $0$ & $0.07$ & $0.29$ & $0.35$ & $0.17$ & $0$ & $0.12$
\\
\hline
$\bel_{\mass_{1 \oplus 2\oplus 3}}$ & $0$  &  $0.07$ & $0.29$ & $0.35$ & $0.54$ & $0.64$ & $1$
\\
\hline
$\pl_{\mass_{1 \oplus 2\oplus 3}}$ & $0$  &  $0.36$ & $0.58$ & $0.46$ & $0.65$ & $0.93$ & $1$
\\
\hline
\end{tabular}
\end{center}

\paragraph{Analysis of the result.}
Dempster's rule of combination treats every user's piece of evidence as a constraint for other users' pieces of evidence. 
For example, if User $1$ is reports that  $S$ is definitely a $\mathtt{Pop}$ song ($\mass_1(\mathtt{Pop}) =1$)  and User $2$  that $S$ is definitely an $\mathtt{R\&B}$ song ($\mass_2(\mathtt{R\&B})=1$),  then the corresponding aggregated mass assignments  yields $\mass_{1 \oplus 2}(\mathtt{Pop}\mbox{-}\mathtt{R\&B}) =1 $, suggesting that   $S$ can be definitely categorized as a member of the greatest common subcategory of $\mathtt{Pop}$ and $\mathtt{R\&B}$. 


Notice that, in the example above, 
individual mass values for $\mathtt{Funk}$ and $\mathtt{Pop}$ are identical, hence one cannot a priori tell which category best describes song $S$. However, the aggregated mass function assigns   $\mathtt{Funk}$ a greater value than  $\mathtt{Pop}$. 
Hence, combining evidence provides us with more information  and allows for a more accurate categorization of the given song $S$.

Notice also that $\mass_1 ( \top ) =0.8$ is a much higher value than $\mass_2 (\top)$ and $\mass_3 (\top)$.  
Saying that song $S$ belongs to $\top$ provides the  least accurate categorization of $S$, 
which implies that User $1$'s response has the smallest impact on the combined evidence. 
Indeed, when combining  mass functions $\mass_1$ and $\mass_2$, if one of them, say $\mass_2 $, is such that  $\mass_2 (\top) =1$, then then the combined mass function $\mass_{1\oplus 2}$  coincides with $\mass_1$,
that is, $\mass_2 $ provides no information. 

Notice that if $\mass_i(\top)=0$ for any  $i=1,2,3$ then $\mass_{1 \oplus 2\oplus3} (\top) =0$. 
Indeed, if $\mathtt{C}$ is any category for which some mass function $\mass$ satisfies $\mass(\mathtt{C}') =0$, for all $ \mathtt{C}'\geq \mathtt{C}$, then, on combining it with any other mass function, the combined mass function $\mass' $ will also satisfy $\mass' (\mathtt{C}' ) =0$ for all $ \mathtt{C}'\geq \mathtt{C}$. 
This shows that categorization based on combined mass is at least as informative as categorization based on any individual piece of evidence. 

Looking at the belief and plausibility functions obtained by combining  evidence, one can see that the highest belief and plausibility (while excluding $\top$) are assigned to  $\mathtt{R\&B}$. 
This is the case because User $3$ provides a more significant piece of evidence than Users $1$ and $2$ and this piece of evidence supports  $S$ being a member of $\mathtt{Pop}\mbox{-}\mathtt{R\&B}$ and $\mathtt{Funk}$. Hence, in the aggregate, the mass values of $\mathtt{Pop}\mbox{-}\mathtt{R\&B}$ and $\mathtt{Funk}$ are higher than the mass values of the other categories. The same reasoning explains why the mass value of $\mathtt{Pop}$ is higher than that of $\mathtt{E}\mbox{-}\mathtt{Pop}$.
Since  $\mathtt{R\&B}$ is a super-category of $\mathtt{Pop}\mbox{-}\mathtt{R\&B}$ and $\mathtt{Funk}$, the value of the combined belief function in $\mathtt{R\&B}$ turns out to be higher than that of the other categories.

\section{Conclusion}
\label{sec:conclusion}
We have introduced a framework which generalizes  basic notions and results of Dempster-Shafer theory from predicates to formal concepts, and 
give preliminary examples showing how some of these notions apply to reasoning and decision-making  under uncertainty for  categorization problems. 

\paragraph{Understanding belief and plausibility functions on concepts.} 
As we have seen in the examples presented in this paper 
the notion of mass function  is a very versatile tool to formalize decision-making problems in a variety of cases. 
In future work, we plan to gain a better understanding of belief and plausibility functions and to apply them in the formalization of concrete situations.

 \paragraph{Towards a logical theory of conceptual evidence.} In this paper we have not pursued an explicitly logical approach; however,  the structures introduced in Section \ref{ssec:DS:structures} lend themselves naturally as bases for models of  an epistemic/probabilistic logic of categories  generalizing   the epistemic logics for Dempster-Shafer theory  introduced in  \cite{ruspini1986logical, godo2003fuzzy}.
 
 \paragraph{Dempster-Shafer theory of concepts and rough concepts.} A key intermediate step to concretely pursue the direction indicated in the previous point can be  the connection currently emerging between Formal Concept Analysis (FCA) and Rough Set Theory (RST) \cite{pawlak}. Indeed, Rough Set Theory is a theory in information science which provides a purely qualitative modelling of incomplete information via upper and lower  definable approximations of given sets. Instances of these definable approximations arise from outer and inner measures induced by probability measures  (cf.~\cite[Section 7.3]{roughconcepts}), which provides a precise way to articulate the affinity between RST and Dempster-Shafer theory. The connections between the two theories have been  investigated in the literature already for some time (cf.~e.g.~\cite{yao1998-nonequivalence}). 
 
The connection between FCA and RST builds on   \cite{conradie2016categories, Tarkpaper}, where  a framework based on formal concepts is introduced which serves as generalized Kripke semantics for an epistemic logic of categories. In \cite{roughconcepts}, it is shown how the basic structures of RST can be represented as special structures based on formal contexts, in a way that `preserves the embedding', as it were, between the natural modal logics of these structures; in \cite{ICLA2019paper},  it is shown how previous proposals for integrating FCA and RST are subsumed by the logic-based proposal of \cite{roughconcepts}.



\bibliographystyle{ACM-Reference-Format}
\bibliography{BIBeusflat2019}

%

%

\end{document}